\documentclass[letterpaper, 10 pt]{ieeeconf}
\IEEEoverridecommandlockouts
\usepackage{cite}
\usepackage{amsmath,amssymb,amsfonts}
\usepackage{mathtools}
\usepackage{bm}
\usepackage{bbold}
\usepackage{graphicx}
\usepackage{textcomp}
\usepackage{xcolor,color}

\usepackage{amsthm}
\usepackage{amsmath}
\usepackage[ruled,vlined]{algorithm2e}
\usepackage[top=0.75in, left=0.75in, right=0.75in, bottom=0.8in]{geometry}
\usepackage{hyperref}
\usepackage{tikz}
\usetikzlibrary{positioning,automata}
\usepackage{subcaption}
\usepackage{empheq}
\usepackage{listings}

\newtheorem{theorem}{Theorem}

\newtheorem{remark}{Remark}

\definecolor{codegreen}{rgb}{0,0.6,0}
\definecolor{codegray}{rgb}{0.5,0.5,0.5}
\definecolor{codepurple}{rgb}{0.58,0,0.82}
\definecolor{backcolour}{rgb}{0.95,0.95,0.92}

\lstdefinestyle{mystyle}{
    backgroundcolor=\color{backcolour},   
    commentstyle=\color{codegreen},
    keywordstyle=\color{magenta},
    numberstyle=\tiny\color{codegray},
    stringstyle=\color{codepurple},
    basicstyle=\ttfamily\footnotesize,
    breakatwhitespace=false,         
    breaklines=true,                 
    captionpos=b,                    
    keepspaces=true,                 
    numbersep=5pt,                  
    showspaces=false,                
    showstringspaces=false,
    showtabs=false,                  
    tabsize=2
}
\newcommand{\x}{\bar{x}}
\newcommand{\X}{\bar{X}}

\begin{document}

\title{\LARGE \bf Equality Constrained Diffusion for Direct Trajectory Optimization}

\author{Vince Kurtz and Joel W. Burdick
\thanks{The authors are with the Department of Civil and Mechanical Engineering, California Institute of Technology, \texttt{\{vkurtz,jwb\}@caltech.edu}}}

\maketitle
\thispagestyle{empty}

\begin{abstract}
The recent success of diffusion-based generative models in image and natural language processing has ignited interest in diffusion-based trajectory optimization for nonlinear control systems. Existing methods cannot, however, handle the nonlinear equality constraints necessary for direct trajectory optimization. As a result, diffusion-based trajectory optimizers are currently limited to shooting methods, where the nonlinear dynamics are enforced by forward rollouts. This precludes many of the benefits enjoyed by direct methods, including flexible state constraints, reduced numerical sensitivity, and easy initial guess specification. In this paper, we present a method for diffusion-based optimization with equality constraints. This allows us to perform direct trajectory optimization, enforcing dynamic feasibility with constraints rather than rollouts. To the best of our knowledge, this is the first diffusion-based optimization algorithm that supports the general nonlinear equality constraints required for direct trajectory optimization.
\end{abstract}

\section{Introduction and Related Work}\label{sec:intro}

Many optimal control and planning methods involve a discrete-time trajectory optimization of the form
\begin{subequations}\label{eq:ocp}
\begin{align}
    \min_{x_k, u_k} ~& \phi(x_K) + \sum_{k=0}^{K-1} \ell(x_k, u_k), \\
    \mathrm{s.t.} ~& x_{k+1} = f(x_k, u_k), \label{eq:ocp:dyn} \\
                   & x_0 = x_{init},\label{eq:ocp:init}
\end{align}
\end{subequations}
where $x_k$ represent states, $u_k$ are control inputs, $\ell$ and $\phi$ are running and terminal costs, $f$ encodes the nonlinear system dynamics, and the initial condition $x_{init}$ is fixed.

Such trajectory optimization problems arise in the context of motion planning, where offline solutions provide a dynamically feasible reference trajectory. They can also form the basis of Model Predictive Control (MPC), where \eqref{eq:ocp} is solved in receding horizon fashion, with the initial condition \eqref{eq:ocp:init} updated online with the latest state estimate.

\begin{figure}
    \centering
    \includegraphics[width=0.8\linewidth]{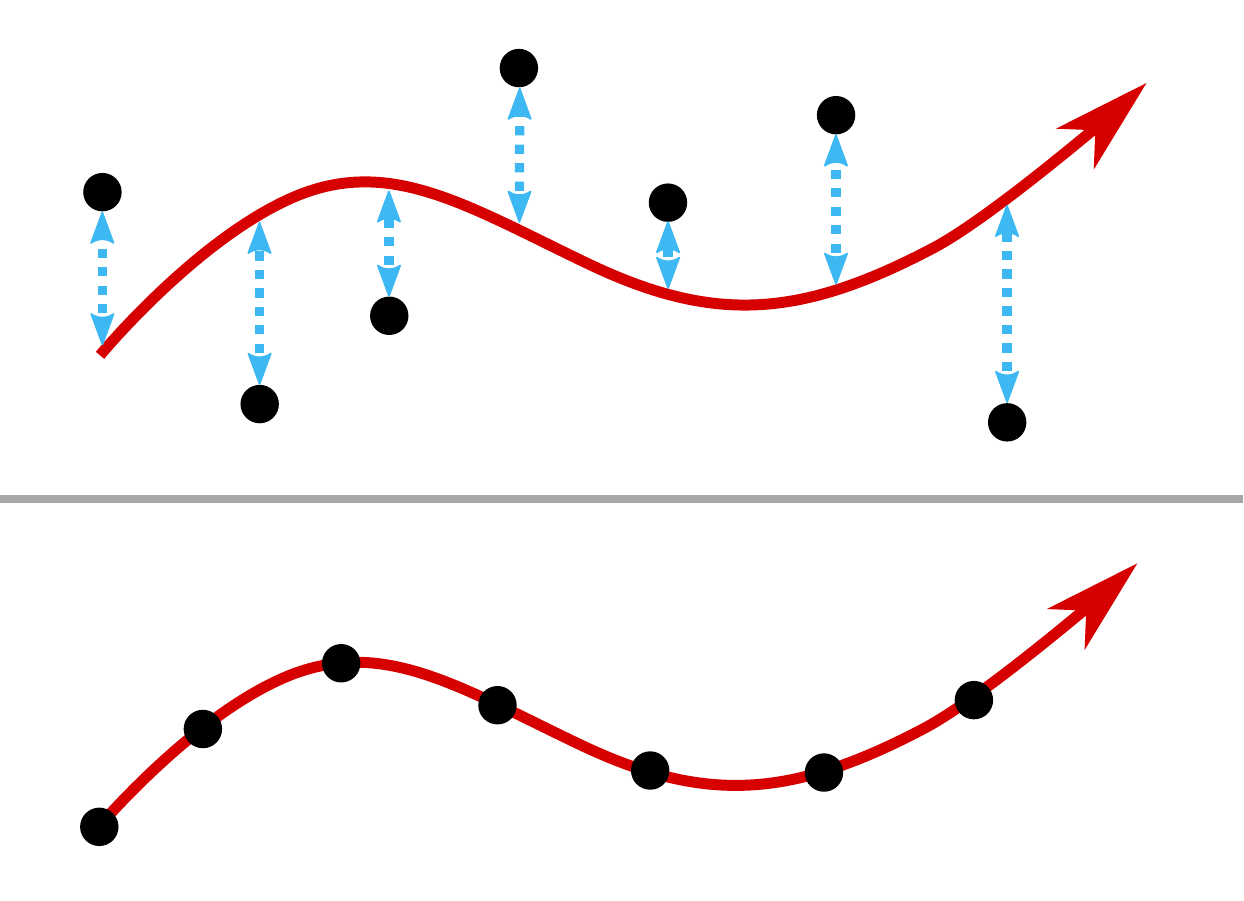}
    \caption{Cartoon illustration of our diffusion-based direct trajectory optimization. In early iterations (top), dynamic feasibility is not required, enabling efficient exploration of the cost landscape. Equality constraints (blue dashed lines) pull the decision variables toward feasibility during the diffusion process, resulting in a final trajectory (bottom) that respects the system dynamics (red curve).}
    \label{fig:hero}
\end{figure}

There are many ways to solve this sort of trajectory optimization \cite{betts1998survey}. Broadly speaking, \textit{shooting methods} eliminate the dynamics constraint $\eqref{eq:ocp:dyn}$ by writing $x_k$ in terms of $u_{0:k-1}$. This results in an unconstrained problem, where each evaluation of the cost involves a forward rollout of the dynamics. On the other hand, \textit{direct methods} use both $x_k$ and $u_k$ as decision variables, enforcing the dynamics with constraints. Direct methods typically leverage a general-purpose non-linear programming (NLP) solver to enforce these constraints, although domain-specific solvers can often obtain improved performance in practice \cite{erez2012trajectory, posa2014direct, kurtz2023inverse}.

While both shooting and direct methods can be effective \cite{betts1998survey}, the direct approach offers several notable advantages. Direct methods are less sensitive than shooting methods. For example, the long-term behavior of a second-order mechanical system changes dramatically with small changes to control torques early in the trajectory. This manifests as a sensitive and poorly-conditioned shooting problem, making optimization difficult. Direct methods do not have this issue, as some level of constraint violation is allowed in early iterations. Direct methods also make it easier to specify an initial guess, support additional state constraints out of the box, and enable more flexible exploration during the optimization process \cite{hargraves1987direct}.

In this paper, we focus in particular on \textit{diffusion-based trajectory optimization}. Diffusion-based trajectory optimization and planning has recently surged in popularity \cite{xue2024full, pan2024model, li2024efficient, carvalho2023motion, jiang2023motiondiffuser,janner2022planning,liang2023adaptdiffuser,ajay2022conditional, song2023loss}, fueled by the success of diffusion-based generative models in computer vision and natural language processing \cite{song2019generative, song2020score ,ho2020denoising, yang2023diffusion}. These techniques frame optimization as a process of sampling from an energy-based probability distribution. This reduces sensitivity to local minima and enables tight connections with behavior cloning techniques like diffusion policy \cite{chi2023diffusion}. 

Nonetheless, existing methods focus primarily on a shooting approach to the optimization problem \cite{pan2024model, xue2024full}, due to the difficulty of enforcing nonlinear dynamics constraints like \eqref{eq:ocp:dyn} in a diffusion process. While various methods for constrained diffusion have been proposed in the literature, including barrier methods \cite{fishman2024diffusion}, reflection methods \cite{fishman2024diffusion, lou2023reflected}, metropolis sampling \cite{fishman2024metropolis}, and L1 penalties \cite{li2024efficient}, none of these methods can handle the (possibly difficult) nonlinear dynamics constraints required for direct trajectory optimization. In addition to high nonlinearity, such equality constraints have a low probability measure, making it exceedingly unlikely that noisy samples, which are key to this method's success, will satisfy the constraint. Approaches that project to the feasible manifold at each step \cite{lelievre2012langevin} come closest to meeting the needs of direct trajectory optimization, but our aim is to enforce dynamics constraints only at convergence rather than at each step of the diffusion process. As we will see, this additional flexibility allows for more effective exploration of the solution space, and can reduce vulnerability to local minima. 

We present a diffusion-based optimization framework that can enforce the difficult nonlinear equality constraints required for direct trajectory optimization. Our proposed approach is heavily inspired by constrained differential optimization \cite{platt1987constrained}, and results in a Langevin diffusion process for both the decision variables and Lagrange multipliers. In addition to proving theoretical convergence properties, we introduce an open-source JAX \cite{jax2018github} implementation and provide a detailed discussion of limitations and areas for future work. To the best of our knowledge, this is the first diffusion-based optimization method that handles general nonlinear equality constraints. 

The remainder of this paper is organized as follows. We present a formal problem statement in Section~\ref{sec:problem_formulation}, followed by background on Langevin dynamics and diffusion-based optimization in Section~\ref{sec:background}. Our main result, a diffusion-based method for equality constrained non-convex optimization, is presented in Section~\ref{sec:main}. We introduce a prototype python/JAX software implementation in Section~\ref{sec:software}, and illustrate the performance of the method with several numerical examples. We discuss the benefits and drawbacks of the approach in Section~\ref{sec:discussion}, and conclude with Section~\ref{sec:conclusion}.
\section{Problem Formulation}\label{sec:problem_formulation}

%\subsection{Notation}
%
%\begin{itemize}
%    \item We use upper-case for random variables ($X$), lower-case for real-values and instantiations $x$.
%    \item Subscript by $t$ for continuous-time stochastic process $X_t$
%    \item Wiener process $dX_t = dW_t$
%    \item Bold for vectors and matrices?
%    \item Elements of matrix $A_{ij}$
%\end{itemize}

In this paper, we consider generic equality-constrained non-linear programs (NLPs) of the following form:
\begin{subequations}\label{eq:nlp}
\begin{align}
    \min_{\x} ~& c(\x), \\
    \mathrm{s.t.} ~& h(\x) = 0, \label{eq:nlp:constraint} \\
                   & \x_L \leq \x \leq \x_U, \label{eq:nlp:bound}
\end{align}
\end{subequations}
where $\x \in \mathbb{R}^n$ are the decision variables, $c : \mathbb{R}^n \to \mathbb{R}$ is a cost function, $h : \mathbb{R}^n \to \mathbb{R}^m$ is a vector of constraints, and $\x_L, \x_U \in [\mathbb{R} \cup \{-\infty, \infty\}]^n$ are lower and upper bounds. We assume that $c(\cdot)$ and $h(\cdot)$ are twice-continuously differentiable, \eqref{eq:nlp:constraint}-\eqref{eq:nlp:bound} are feasible, and the NLP is not unbounded.

Clearly, the direct trajectory optimization problem \eqref{eq:ocp} is a special case of this NLP, where we have
\begin{gather}
    \x = [u_0, \dots, u_{K-1}, x_0, \dots, x_K],
\end{gather}
\begin{gather}
    c(\x) = \phi(x_K) + \sum_{k=0}^{K-1} \ell(x_k, u_k), \\
    h(\x) = \begin{bmatrix} \vdots \\ x_{k+1} - f(x_k, u_k) \\ \vdots \\ x_0 - x_{init} \end{bmatrix}.
\end{gather}
Additionally, inequality constraints of the form $g(\x) \leq 0$ can be trivially added with the addition of slack variables \cite{nocedal1999numerical}.

Nonetheless, direct trajectory optimization problems are somewhat unique among NLPs in the sense that \textit{the constraints are the hard part}. While the cost $c(\x)$ is typically composed of convex quadratic tracking terms, the constraints $h(\x)$ encode the nonlinear system dynamics. In this paper, we present a diffusion-based optimization framework that can handle such nonlinear equality constraints.

\section{Background}\label{sec:background}

\subsection{Langevin Sampling}\label{sec:background:langevin}

The theory of Langevin sampling \cite{roberts1996exponential} grounds diffusion-based generative models \cite{song2019generative, song2020score, ho2020denoising, yang2023diffusion}. In this setting, we generate samples from a probability distribution
\begin{equation}
    Y \sim p(Y)
\end{equation}
via a Langevin stochastic differential equation (SDE),
\begin{equation}
    dY_t = \frac{1}{2}\nabla_y \log p(Y_t) dt + dW_t,
\end{equation}
where $W_t$ denotes a standard Wiener process. We typically simulate this SDE with an Euler-Maruyama discretization:
\begin{equation}
    y_{i+1} \gets y_i + \alpha \nabla_y \log p(y_i) + \sqrt{2\alpha} \epsilon_i, \quad \epsilon_i \sim \mathcal{N}(0, I).
\end{equation}
\begin{remark}
    Throughout this paper, we use capital letters (e.g., $Y$) to refer to random variables and lower-case (e.g., $y$) to refer to realizations and real-valued variables.
\end{remark}
The \textit{score} $\nabla_y \log p(y)$ holds particular importance. Learning-based generative modeling techniques train a neural network to approximate the score of a data distribution \cite{ho2020denoising,song2020score,song2019generative}, while sampling-based trajectory optimization methods like model-predictive path integral control (MPPI) \cite{williams2017model} approximate a score with model rollouts \cite{xue2024full}.

\subsection{Unconstrained Diffusion-Based Optimization}\label{sec:background:unconstrained}

Diffusion-based generative modeling is growing in prominence as a tool for unconstrained optimization, e.g.,
\begin{equation}\label{eq:unconstrained_opt}
    \min_{\x} c(\x).
\end{equation}
In this setting, we define a time-varying target distribution
\begin{equation}
    p_t(\X) = \frac{1}{\eta}\exp\left(- c(\X)/\sigma(t)^2\right),
\end{equation}
where $\sigma(t)$ decays to zero and $\eta$ is a normalizing constant. The corresponding Langevin dynamics are given by
\begin{equation}
    d\X_t = -\frac{1}{2}\frac{\nabla_{\x} c(\X_t)}{\sigma(t)^2} dt + dW_t,
\end{equation}
where the score $\nabla_{\x} \log p_t(\X) = -\nabla_{\x} c(\X) / \sigma(t)^2$ is independent of the normalizing constant $\eta$ \cite{song2019generative}. Furthermore, this SDE can be equivalently expressed as
\begin{equation}
    d\X_t = -\frac{1}{2}\nabla_{\x} c(\X) dt + \sigma(t) dW_t,
\end{equation}
emphasizing that the noise level decays with the decreasing magnitude of $\sigma(\cdot)$ over the course of the diffusion process.

Annealing the noise level $\sigma(\cdot)$ concentrates probability mass at global minimizers of $c(\x)$:
\begin{theorem}[\cite{pan2024model}]
    Let $c(\cdot)$ be such that level sets of $c(\x)$ satisfy the volumetric conditions\footnote{These conditions establish that the volume of sub-level sets of $c$ must be bounded above and below by polynomial inequalities, and essentially establish that $c$ is well-behaved and $\x^*$ exists. We refer the interested reader to \cite{pan2024model} for further details.} of \cite[Proposition 2]{pan2024model}, and let $\X \sim \frac{1}{\eta}\exp(-c(\X)/\sigma^2)$ for some $\sigma \geq 0$. Assume without loss of generality that $\min_{\x} c(\x) = 0$ and $\x^* = \arg \min c(\x)$. Then $c(\X)$ converges in probability to $c(\x^*)$ as $\sigma \to 0$.
\end{theorem}
This property relates diffusion-based optimization closely with Langevin simulated annealing \cite{borysenko2021coolmomentum, chak2023generalized}, a global optimization method that enjoys precisely characterized convergence properties \cite{bras2024convergence}.

\section{Main Results}\label{sec:main}

In this section we present our main result: an extension of diffusion-based optimization to equality-constrained NLPs \eqref{eq:nlp}.

To begin, we note that there are many well-established methods for dealing with the bound constraints $\x_L \leq \x \leq \x_U$. These include log-barrier methods \cite{nocedal1999numerical, fishman2024diffusion}, reflection methods \cite{lou2023reflected}, and metropolis corrections \cite{fishman2024metropolis}. This allows us to focus narrowly on equality constrained problems of the form
\begin{subequations}\label{eq:simple_constrained_opt}
\begin{align}
    \min_{\x} &~ c(\x), \\
    \mathrm{s.t.} &~ h(\x) = 0,
\end{align}
\end{subequations}
where, again, the inequality constraints $h(\x)$ capture the nonlinear system dynamics and present the central challenge.

An obvious approach, following Section~\ref{sec:background:unconstrained}, would be to define a constrained target distribution,
\begin{equation}
    p_t(\X) = \frac{1}{\eta} \exp\left(-c(\X)/\sigma(t)^2\right) \mathbb{1}_{h(\X) = 0},
\end{equation}
where $\mathbb{1}_{h(\X) = 0} = \begin{cases}
    1 & h(\X) = 0 \\
    0 & \mathrm{otherwise}
\end{cases}$ enforces the constraint. But this approach introduces several issues. For one, the score $\nabla_{\x} \log p_t(\X)$ is not well-defined for this distribution. A sampling-based approximation, as pursued in \cite{pan2024model}, might provide a way forward mathematically, but obtaining feasible realizations via random sampling is extremely unlikely.

Instead, we propose generating solutions to \eqref{eq:simple_constrained_opt} by simulating the following diffusion process:
\begin{subequations}\label{eq:constrained_diffusion}
\begin{empheq}[box=\fbox]{align}
    & \begin{multlined}
    d\X_t = - \frac{1}{2}\Big[ \nabla_{\x} c(\X_t) + \lambda_t^T \nabla_{\x} h(\X_t) \\ \quad\quad\quad\quad + \mu \nabla_{\x} h(\X_t)^T h(\X_t)\Big]dt + \sigma(t) dW_t,
    \end{multlined} \label{eq:constrained_diffusion:x}\\
    & d\lambda_t = \mu h(\X_t) dt. \label{eq:constrained_diffusion:lambda}
\end{empheq}
\end{subequations}
Here $\mu > 0$ is a positive constant, $\sigma(t)$ is a decaying noise schedule, and $\lambda_t$ takes the role of a Lagrange multiplier.

This diffusion process is a stochastic extension of constrained differential optimization \cite{platt1987constrained}, and is analogous to an augmented Lagrangian method \cite{nocedal1999numerical}. Specifically, the drift term in \eqref{eq:constrained_diffusion:x} is the gradient of an augmented-Lagrangian-style cost
\begin{equation}
    c(\x) + \lambda^Th(\x) + \frac{\mu}{2} \|h(\x)\|_2^2,
\end{equation}
where the value of the Langrange multipliers $\lambda$ are determined adaptively via \eqref{eq:constrained_diffusion:lambda}.

The convergence of this SDE can be characterized more formally as follows:

\begin{theorem}\label{thm:convergence}
    Let $\X_t$ follow the diffusion process \eqref{eq:constrained_diffusion} with penalty coefficient $\mu > 0$ and noise schedule $\sigma: \mathbb{R}^+ \to \mathbb{R}^+$ such that $\lim_{t \to \infty} \sigma(t) = 0$. Then there exists a $\mu^*$ such that if $\mu > \mu^*$, $\X_t$ converges to a constrained minimizer of \eqref{eq:simple_constrained_opt} in probability as long as $\X_t$ remains bounded and in a region $R$ surrounding a constrained minimizer. Additionally, as $\mu \to \infty$, $\X_t$ is globally convergent in probability to a constrained minimizer of \eqref{eq:simple_constrained_opt}.
\end{theorem}
\begin{proof}[Proof (sketch)]
    Recall that the Lagrangian optimality conditions for \eqref{eq:simple_constrained_opt} are given by
    \begin{subequations}\label{eq:lagrange_conditions}
    \begin{align}
        & \partial \mathcal{L}(\x, \lambda) / \partial \x = 0, \\
        & \partial \mathcal{L}(\x, \lambda) / \partial \lambda = 0,
    \end{align}
    \end{subequations}
    where $\mathcal{L}(\x, \lambda) = c(\x) + \lambda^T h(\x)$ is the Lagrangian. We will show that the diffusion process \eqref{eq:constrained_diffusion} converges in probability to values satisfying these conditions.
    Let
    \begin{equation}
        v(\X_t) = \nabla_{\x} c(\X_t) + \lambda_t^T \nabla_{\x} h(\X_t) + \mu \nabla_{\x} h(\X_t)^T h(\X_t)
    \end{equation}
    denote the ``velocity'' of the decision variables in \eqref{eq:constrained_diffusion}. Then define a corresponding ``energy''
    \begin{equation*}
        E_t = e(\X_t) = \frac{1}{2}v(\X_t)^T v(\X_t) + \frac{1}{2}h(\X_t)^T h(\X_t),
    \end{equation*}
    which is composed of a ``kinetic'' followed by a ``potential'' energy term. Note that $E_t$ is always non-negative, and $E_t = 0$ enforces the Lagrangian optimality conditions \eqref{eq:lagrange_conditions}.
   
    By It\^{o}'s lemma, we have that
    \begin{multline}
        dE_t = \left[\nabla_{\x} e(\X_t)^T v(\X_t) + \frac{1}{2}Tr[\sigma(t)^2 \nabla_{\x}^2 e(\X_t)]\right]dt \\
        + \nabla_{\x} e(\X_t)^T \sigma(t) dW_t,
    \end{multline}
    where $\nabla_{\x}^2 e(\X_t)$ denotes the Hessian of the energy function. As $t \to \infty$, the drift term dominates and we have
    \begin{equation}
        dE_t = \nabla_{\x} e(\X_t)^T v(\X_t) dt,
    \end{equation}
    since $\sigma(t) \to 0$. This drift term is then given \cite{platt1987constrained} by
    \begin{equation}
        dE_t = - v(\X_t) ^T A v(\X_t),
    \end{equation}
    where 
    \begin{equation}\label{eq:energy_matrix}
        A_{ij} = \frac{\partial^2 c}{\partial \x_i \partial \x_j} + \lambda \frac{\partial^2 h}{\partial \x_i \partial \x_j} + \mu \frac{\partial h}{\partial \x_i} \frac{\partial h}{\partial \x_j} + \mu h \frac{\partial^2 h}{\partial \x_i \partial \x_j}.
    \end{equation}
    The matrix $A$ is positive definite in some region $R$ surrounding constrained minima for all $\mu > \mu^*$ \cite{platt1987constrained}, providing convergence in probability to $E_t = 0$, and thus to a constrained minimizer.

    The case where $\mu \to \infty$ follows similarly from the well-known optimization result that minimizers of
    \begin{equation}
        c(x) + \frac{\mu}{2} \|h(\x)\|_2^2
    \end{equation}
    converge globally to local solutions of \eqref{eq:simple_constrained_opt} as $\mu \to \infty$ \cite{nocedal1999numerical}.
\end{proof}

While Theorem~\ref{thm:convergence} only guarantees global convergence as $\mu \to \infty$, the Lagrange multipliers $\lambda$ allow us to achieve reliable convergence with much smaller penalties in practice, as illustrated in the following section. This improves the numerical conditioning of the problem, allowing for a larger step-size ($\alpha$) when simulating the SDE \eqref{eq:constrained_diffusion}.

\section{Software implementation in JAX}\label{sec:software}

We provide a prototype implementation\footnote{\texttt{\url{https://github.com/vincekurtz/drax}}} of our proposed approach in JAX, a software framework for massively parallel scientific computing on GPUs and other hardware accelerators \cite{jax2018github}. JAX supports forward and reverse-mode automatic differentiation on numpy arrays. This means that users need only define the discrete-time system dynamics $x_{k+1} = f(x_k, u_k)$, a running cost $\ell(x_k, u_k)$, and a terminal cost $\phi(x)$, as shown in the following code snippet:

\begin{lstlisting}[language=Python]
class PendulumSwingup(OptimalControlProblem):
    ...
    def dynamics(self, x, u):
        theta, theta_dot = x
        tau = u[0]
        ...
        xdot = jnp.array([theta_dot, theta_ddot])
        return x + self.dt * xdot

    def running_cost(self, x, u):
        return self.dt * 0.01 * jnp.sum(u**2)

    def terminal_cost(self, x):
        theta, theta_dot = x
        return 10 * theta**2 + 1 * theta_dot**2
\end{lstlisting}

Our software translates such a direct trajectory optimization problem \eqref{eq:ocp} into an equality-constrained NLP \eqref{eq:nlp}. We then solve the optimization by simulating \eqref{eq:constrained_diffusion} with Euler-Maruyama discretization, geometrically decaying noise level $\sigma(t)$, and user-specified penalty coefficient $\mu$. We enforce bound constraints \eqref{eq:nlp:bound} with a simple log-barrier \cite{fishman2024diffusion}.

Here we consider two example nonlinear trajectory optimization problems. The first, an inverted pendulum, is tasked with swinging to an upright position. The dynamics are
\begin{equation}
    x_{k+1} = x_{k} + \delta t 
    \begin{bmatrix}
        \dot{\theta}_k \\
        \left(u_k - mgl\sin(\theta_k - \pi) \right) / (m l^2)
    \end{bmatrix},
\end{equation}
where $x_k = [\theta_k, \dot{\theta}_k]$ is composed of the angle and angular velocity of the pendulum, which has mass $m$, length $l$ and gravitational acceleration $g$. We denote use a timestep of $\delta t = 0.1$ seconds. Strict torque limits ($u_k \in [-1, 1]$) require the pendulum to swing back and forth before reaching the upright.

A phase portrait of the solution is shown in Fig.~\ref{fig:pendulum_phase_portrait}. The process starts with an infeasible initial guess (top), with $x_k$ scattered at random. During the diffusion, \eqref{eq:constrained_diffusion} drives $\x$ to both feasibility and optimality, resulting in a smooth and dynamically feasible swingup trajectory (middle) where control torques alternate between upper and lower limits (bottom).

\begin{figure}
    \centering
    \includegraphics[width=0.9\linewidth]{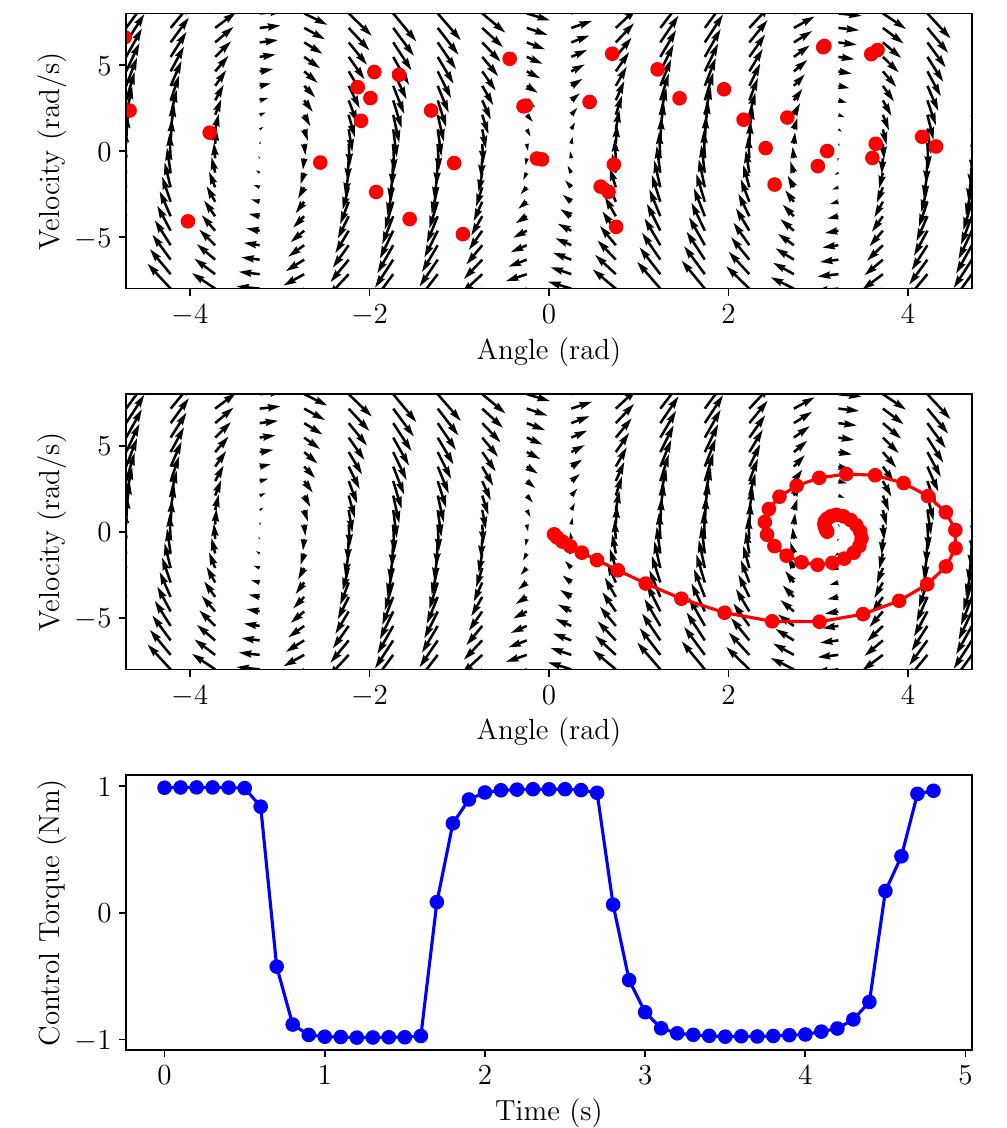}
    \caption{Dynamically infeasible initial guess (top), optimized trajectory (middle), and optimized control sequence (bottom) for a pendulum swingup trajectory generated with our proposed approach. The Langevin SDE \eqref{eq:constrained_diffusion} drives the samples to both feasibility and optimality during the diffusion process, without requiring feasibility at each iteration.}
    \label{fig:pendulum_phase_portrait}
\end{figure}

Fig.~\ref{fig:pendulum_convergence} plots convergence for various penalty values $\mu$. Specifically, we plot the constraint violation $\|h(\x)\|_2^2$ at each iteration. With very low values ($\mu = 0.01$), the diffusion process does not converge to a dynamically feasible solution. This supports the theory behind Theorem~\ref{thm:convergence}, as there is a minimum $\mu^*$ that enforces positive definiteness of the energy matrix $A$ in the neighborhood of constrained minima \eqref{eq:energy_matrix}. Nonetheless, even modest values of $\mu$ above this threshold are sufficient to drive \eqref{eq:constrained_diffusion} to a feasible solution.

\begin{figure}
    \centering
    \includegraphics[width=0.9\linewidth]{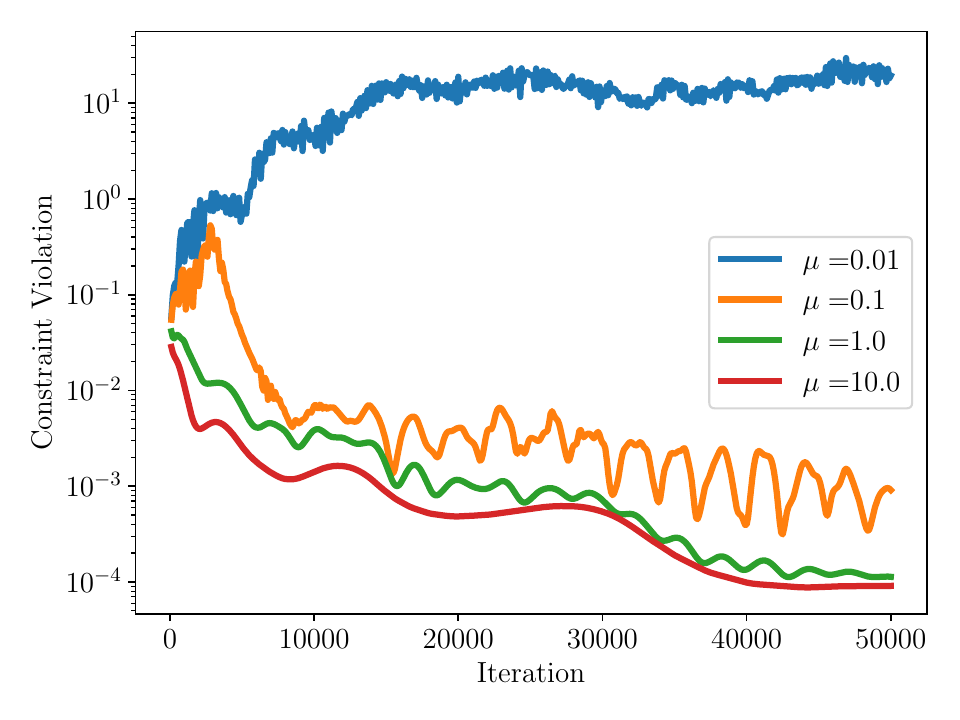}
    \caption{Equality constraint violations for the pendulum swingup task with various penalty parameters $\mu$. While the lowest $\mu$ value does not converge to a dynamically feasible solution, relatively small values of $\mu$ are sufficient to enforce dynamic feasibility by the end of the diffusion process, thanks to the contribution of the Lagrange multipliers in \eqref{eq:constrained_diffusion:lambda}.}
    \label{fig:pendulum_convergence}
\end{figure}

In the second example, a planar robot with unicycle dynamics is tasked with driving around a U-shaped obstacle to reach a goal (Fig.~\ref{fig:bug_trap}). The system dynamics are
\begin{equation}
    x_{k+1} = x_k + \delta t \begin{bmatrix}
        v_k \cos(\theta_k) \\
        v_k \sin(\theta_k) \\
        \omega_k
    \end{bmatrix},
\end{equation}
where $x_k = [p^x_k, p^y_k, \theta_k]$ is the position and orientation of the vehicle and $u_k = [v_k, \omega_k]$ are linear and angular velocity commands. The terminal cost encourages the robot to reach a goal position (green star), while the running cost imposes a harsh penalty on obstacle proximity (red shading).
This example is particularly difficult because it induces a large local minimum where the robot is stuck inside the maze. Standard optimization methods like gradient descent and BFGS get stuck in this local minimum, as illustrated in Fig.~\ref{fig:bug_trap}, while our direct diffusion method does not.

\begin{figure*}
    \centering
    \begin{subfigure}{0.3\linewidth}
        \includegraphics[width=\linewidth]{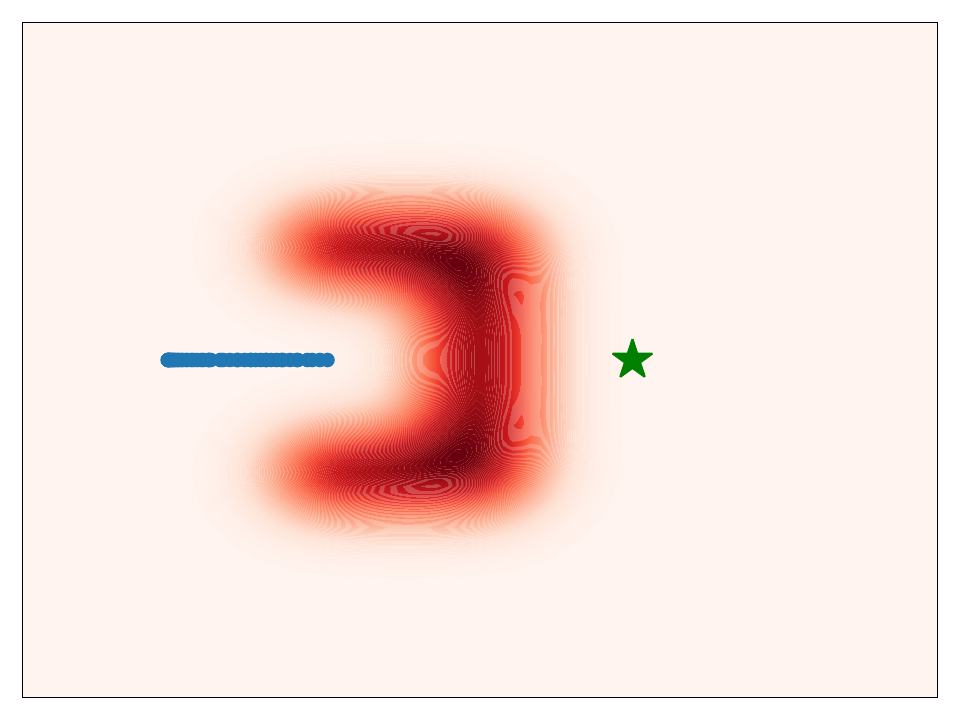}
        \caption{Gradient Descent \cite{platt1987constrained}}
        \label{fig:bug_trap:gd}
    \end{subfigure}
    \begin{subfigure}{0.3\linewidth}
        \includegraphics[width=\linewidth]{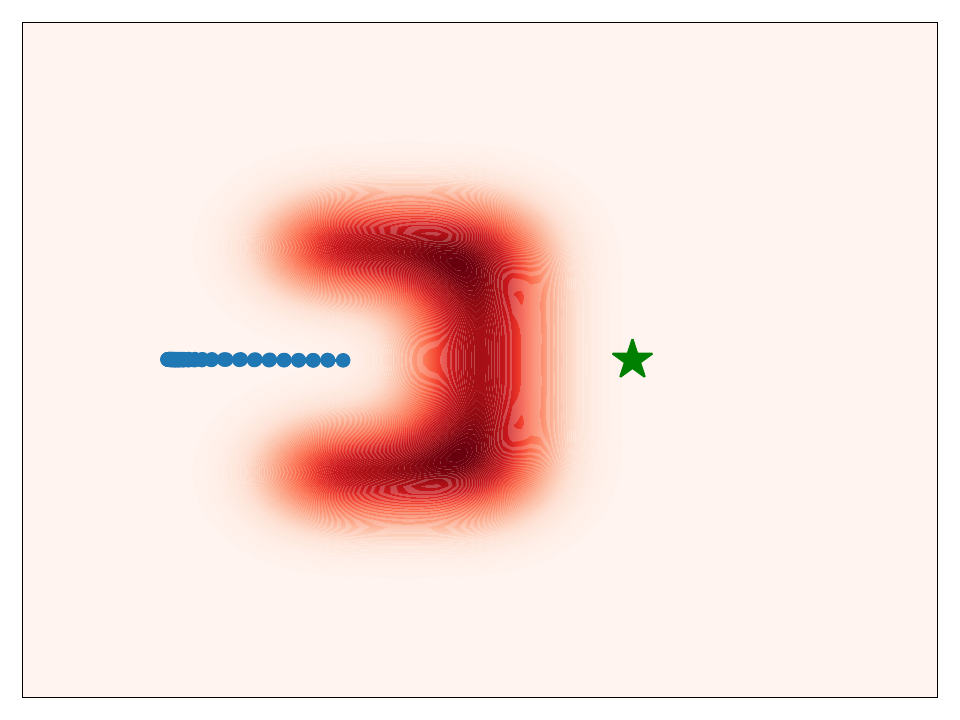}
        \caption{Quasi-Newton (BFGS) \cite{nocedal1999numerical}}
        \label{fig:bug_trap:bfgs}
    \end{subfigure}
    \begin{subfigure}{0.3\linewidth}
        \includegraphics[width=\linewidth]{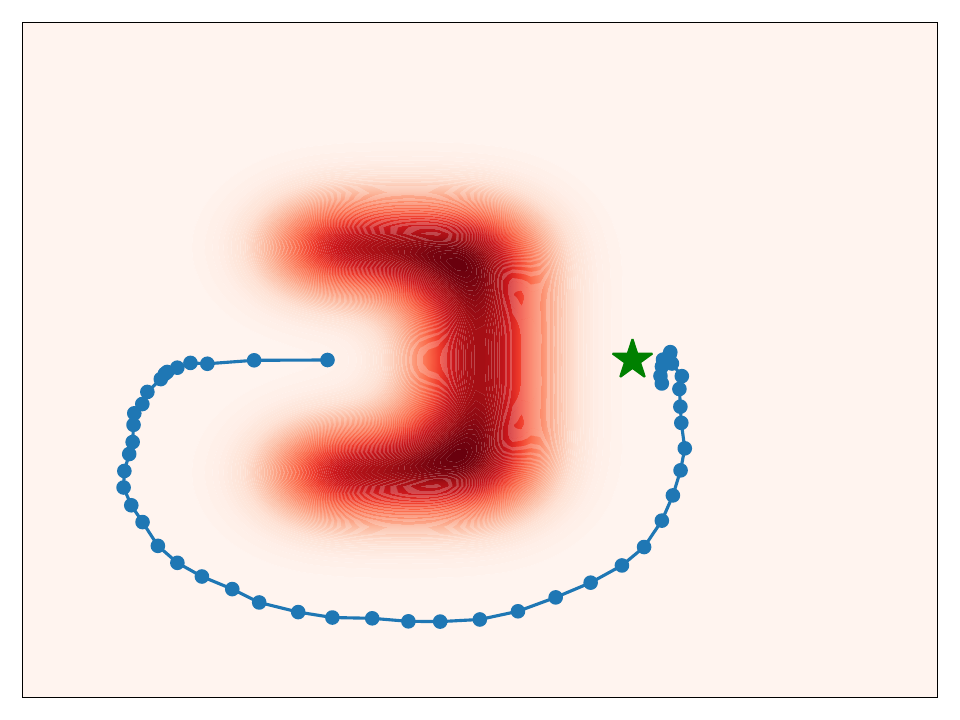}
        \caption{Direct Diffusion (ours)}
        \label{fig:bug_trap:ours}
    \end{subfigure}
    \caption{A planar robot with unicycle dynamics is tasked with escaping a U-shaped maze (red) to reach a goal position (green star). Diffusion-based direct optimization escapes the local minimum, while BFGS and gradient descent do not.}
    \label{fig:bug_trap}
\end{figure*}

Interestingly, recent work on shooting-based diffusion also reports vulnerability to such local minima \cite[Fig.~4b]{pan2024model}. We hypothesize that the structure of the direct optimization problem, where early iterates are allowed to explore infeasible solutions, makes it easier to overcome such challenges. 

Finally, another advantage of our approach is that it involves only simple, GPU-friendly operations that can be trivially parallelized: the SDE \eqref{eq:constrained_diffusion} does not require matrix factorizations or other complex linear algebra operations that tend to be slow on hardware accelerators. This, combined with JAX support for automatic vectorization and parallelization, makes it easy to solve many problems in parallel. 

We solve the pendulum swingup problem in parallel from $N$ different initial conditions, and report total solve times (after 20,000 iterations) for various $N$ in Table~\ref{tab:parallel_solve_times}. While increasing $N$ does increase solve times modestly due to computational overhead, we are able to solve several thousand problem instances in a matter of seconds. All experiments were performed on an Nvidia RTX 4070 desktop GPU. 

\begin{table}[]
    \centering
    \begin{tabular}{c|c c c c c c c}
         $N$      & 64 & 256 & 1024 & 2048 & 4096 & 8192 & 16384 \\
         \hline
         Time (s) & 0.68 & 0.75 & 1.04 & 1.49 & 2.12 & 3.53 & 7.95
    \end{tabular}
    \caption{Pendulum solve times from $N$ randomized initial conditions.}
    \label{tab:parallel_solve_times}
\end{table}

\section{Discussion and Limitations}\label{sec:discussion}

A key advantage of direct diffusion-based trajectory optimization appears to be more reliable convergence to better local minima. While Theorem~\ref{thm:convergence} establishes conditions for convergence to a local optimum, proving global optimality---or more realistically, the conditions under which global optimality is possible---remains a significant challenge. Nonetheless, a rich literature on Langevin simulating annealing \cite{bras2024convergence} may provide a promising starting point for doing so.

Practically speaking, the need for dynamics gradients (e.g., $\nabla_x f(x, u)$) presents a significant limitation. It is difficult to obtain accurate analytical gradients for many systems of interest, especially mechanical systems with contacts. Fortunately, sampling-based gradient estimation \cite{suh2022differentiable} provides a straightforward alternative: our prototype JAX implementation includes support for such estimates. An interesting open question is how sampling-based gradient estimates compare to the Monte-Carlo score estimates used in sampling-based trajectory optimization methods like MPPI \cite{xue2024full}.

Diffusion-based trajectory optimization, whether direct or shooting-based, is not particularly efficient compared to standard NLP solvers. In large part this is due to a lack of second-order information: modern NLP solvers take large Newton-style steps, enabling rapid convergence compared to gradient descent \cite{nocedal1999numerical}. Can we leverage second-order information to accelerate the diffusion process, while maintaining the desirable properties of the direct diffusion framework? Unfortunately, straightforward application of quasi-Newton methods like BFGS is not feasible, as the additional noise in the diffusion process breaks the step-to-step Hessian approximation. 

The choice of annealing schedule $\sigma(t)$ provides another important area for future research. We found that a geometrically decaying sequence worked fairly well, but there may be other more desirable choices. Could the noise level be adapted on the fly? Could we generalize $\sigma$ to a full-rank positive definite matrix $\Sigma$, resulting in an optimization procedure akin to pre-conditioned Langevin sampling \cite{titsias2024optimal}?

Finally, combining these results with learning-based diffusion methods presents a promising area for future work. For instance, can we combine diffusion-based direct trajectory optimization with expert demonstrations, as in \cite{pan2024model}? Could we learn scores associated with the dynamics constraints, then re-use them for different tasks/costs? Regardless of the approach, the simple algebra and embarrassingly parallel nature of direct diffusion-based trajectory optimization makes it a promising candidate for integration with learning methods. 

\section{Conclusion}\label{sec:conclusion}

We presented a diffusion-based method for direct trajectory optimization, where equality constraints are enforced via Lagrange multipliers that flow alongside the decision variables in an extended Langevin diffusion process. In addition to characterizing the convergence properties of this algorithm, we presented several numerical examples and introduced an open-source implementation in JAX. Future work will focus on extending these methods to the online MPC setting, principled incorporation of sampling-based score estimates, and acceleration with second-order information.

\bibliographystyle{IEEEtran}
\bibliography{references}

\end{document}